\documentclass{article}

\usepackage[]{geometry}
\usepackage{authblk}
\usepackage{natbib}
\setcitestyle{authoryear,open={(},close={)}}

\usepackage{amsmath}
\usepackage{amsfonts}
\usepackage{amsthm}
\usepackage{mathtools}
\usepackage{mymacros}
\usepackage{cleveref}

\usepackage{tikz}
\usetikzlibrary{positioning, decorations.markings, shapes, decorations.pathreplacing,calligraphy}

\newcommand{\cond}{\cP^{G'}\supset \cP^G}

\newcommand{\za}{\text{flu}}
\newcommand{\zb}{\text{hayfever}}
\newcommand{\xa}{\text{muscle pain}}
\newcommand{\xb}{\text{congestion}}

\newenvironment{bff}[1][]
{\par 
\medskip 
\noindent \textbf{#1}}{
    \par
    \medskip
    \noindent}

\tikzstyle{neuron} = [circle,draw, minimum size=.63cm, inner sep=0, ]
\tikzstyle{iarrow} = [->]
\tikzstyle{aarrow} = [->, red]
\tikzstyle{farrow} = [<-]

\title{Inversion of Bayesian Networks}

\author[1]{Jesse van Oostrum}
\author[2]{Peter van Hintum}
\author[1, 3, 4]{Nihat Ay}
\affil[1]{Institute for Data Science Foundations, Hamburg University of Technology, Hamburg, Germany}
\affil[ ]{\textit {jesse.van@tuhh.de}}
\affil[2]{New College, University of Oxford, Oxford, UK}
\affil[3]{Leipzig University, Leipzig, Germany}
\affil[4]{Santa Fe Institute, Santa Fe, USA}

\date{}

\begin{document}

\maketitle

\begin{abstract}
    Variational autoencoders and Helmholtz machines use a recognition network (encoder) to approximate the posterior distribution of a generative model (decoder). In this paper we establish some necessary and some sufficient properties of a recognition network so that it can model the true posterior distribution exactly. These results are derived in the general context of probabilistic graphical modelling / Bayesian networks, for which the network represents a set of conditional independence statements. We derive both global conditions, in terms of d-separation, and local conditions for the recognition network to have the desired qualities. It turns out that for the local conditions the \emph{perfectness} property (for every node, all parents are joined) plays an important role.  
\end{abstract}

\section{Introduction}
A generative model is a set of probability distributions that models the distribution of observed and latent variables.  
Generative models are used in many  machine learning applications. One is often interested in inferring the latent variable on the basis of an observation, i.e.\ obtaining the posterior distribution. For complex generative models it is often hard to calculate the posterior distribution analytically. The field of variational Bayesian inference \citep{wainwright2008graphical} studies different ways of approximating the true posterior. One approach within this field is called \textit{amortised inference} \citep{gershman2014amortized}. This approach distinguishes itself through using one set of parameters for recognition that is optimised over multiple data points. This can be contrasted with ``memoryless" inference algorithms, such as the message passing algorithm \citep{pearl1982reverend,cowell1999}, which finds a separate set of parameters for every data point. Both the variational autoencoder (VAE) \citep{kingma2013auto} and the Helmholtz machine \citep{dayan1995helmholtz} are examples of amortised inference. In their most general form, these consist of a Bayesian network that is used to model the generative distribution. A second network, called the recognition model, is used to model the posterior distribution. Both these networks have the same set of nodes, namely the union of the observed and latent variables. In the generative network the arrows point from the latent to the observed nodes but in the recognition network it is the other way around. The recognition network is therefore in some sense an inversion of the generative network. In many applications, one simply reverses the direction of the edges of the generative network to obtain the recognition network. However, as the simple example in Figure \ref{figure:invertingEdgesInsufficient} shows, this does not guarantee that the recognition model is actually able to model the true posterior distribution of the generative model. In this paper, we establish some necessary and some sufficient properties of the recognition network such that we do have this guarantee. 
We first discuss these properties in terms of d-separation, subsequently in terms of perfectness, and finally in terms of single edge operations using the Meek conjecture \citep{meek1997graphical}.\\

\begin{figure}[ht]
    \centering
    \begin{tikzpicture}
        \node (x) {
            \begin{tikzpicture}
                \node (a) at (0,0) [neuron] {$x$};
                \node (b) at (-.7,1) [neuron] {$z_1$};
                \node (c) at (.7,1) [neuron] {$z_2$};
                \node (g) at (-1.3, 1.8) {$G$};
                \draw[->] (b) to (a);
                \draw[->] (c) to (a);
            \end{tikzpicture}
        };
        \node (y) [right=of x] {
            \begin{tikzpicture}
                \node (a) at (0,0) [neuron] {$x$};
                \node (b) at (-.7,1) [neuron] {$z_1$};
                \node (c) at (.7,1) [neuron] {$z_2$};
                \node (g) at (-1.3, 1.8) {$G'$};
                \draw[->] (a) to (b);
                \draw[->] (a) to (c);
            \end{tikzpicture}
        };
    \end{tikzpicture}
    \caption{Pair of DAGs $G$ (left) and $G'$ (right) where $G'$ is obtained by reversing the direction of the edges in $G$. The variables $z_1, z_2$ represent the latent variables and $x$ the observed variable. The distribution $p$ such that $z_1, z_2$ are independent $\mathrm{Bernoulli}(0.5)$ and $x = z_1 + z_2 \mod 2$ can be modelled by $G$, but the conditional distribution $p_{z_1,z_2|x}$ cannot be modelled by $G'$. 
    }
    \label{figure:invertingEdgesInsufficient}
    \end{figure}
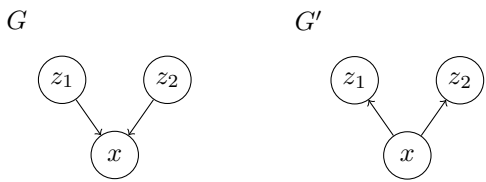

In practice, one often puts further restrictions on the probability distributions the networks can model. For example, the distribution of an individual node can be required to be Gaussian. The mean (and variance) are in this case a function of the values of the parent nodes. We discuss the case of a restricted set of probability distributions in the last part of the results section. \\

The question of finding a sparse $G'$ that can approximate the posterior distribution of the generative model well is also studied from a more practical perspective, using methods from machine learning. One can use a sparsity prior when determining the recognition model, to encourage that only the edges really necessary for modelling the posterior are added. \citet{lowe2022amortized,louizos2017learning,molchanov2019doubly} present several approaches. \\

Markov equivalence is a property of a pair of Bayesian networks that indicates that they encode the same set of conditional independence statements \citep{verma1990equivalence,flesch2007markov}. A generalisation of this, which we will call \emph{Markov inclusion}, is when the set of conditional independence statements encoded in one graph is a subset of the conditional independence statements encoded in the other graph \citep{castelo2003inclusion}. We will see in Proposition \ref{connectVisibleNodes} that the results in this paper can also be viewed as describing under which conditions one Bayesian network is Markov inclusive of another. \\

\citet{webb2018faithful} deal with a closely related problem. They present an algorithm for inverting the generative network that gives a recognition network with the desired properties. While the present article also deals with the algorithmic aspects, it puts more emphasis on the conditions for the recognition network to have the desired properties. 
The authors were unaware of the publication \citep{webb2018faithful} prior to the acceptance of the present paper.

\subsection{Example}
Before giving a formal definition of the problem, we illustrate the topic of this paper by an intuitive example that provides context for the rest of the paper. Consider the generative model for diseases and their symptoms in Figure \ref{fig:generativeModelDiseases}.
\begin{figure}[!ht] 
    \tikzset{tnode/.style={rounded rectangle, draw, minimum width = 2cm}}
    \tikzset{parrow/.style={dotted}}
    \centering
    \begin{tikzpicture}
        \pgfmathsetmacro{\gridspace}{1.5}
        \pgfmathsetmacro{\xspace}{0.9}
        \node (a) {
            \begin{tikzpicture}
                \node [tnode] (xa) at (0,0) {\xa};
                \node [tnode] (za) at (1 * \xspace * \gridspace, 1 * \gridspace) {\za};
                \node [tnode] (xb) at (2 * \xspace * \gridspace, 0) {\xb}; 
                \node [tnode] (zb) at (3 * \xspace * \gridspace, 1 * \gridspace) {\zb};
                \draw [->] (za) to (xa);
                \draw [->] (za) to (xb);
                \draw [->] (zb) to (xb);
            \end{tikzpicture}
        };
        
    \end{tikzpicture}
    \caption{Example from: \cite{koller2009probabilistic} }
    \label{fig:generativeModelDiseases}
\end{figure}
Our goal is to find a model to perform inference on this generative model, i.e.\ to find the posterior distribution $P(\text{diseases}\mid\text{symptoms})$. Naively, one could reverse the edges of the generative model to obtain the recognition model in Figure \ref{naiveRecMod}.

\begin{figure}[!h]
    \tikzset{tnode/.style={rounded rectangle, draw, minimum width = 2cm}}
    \tikzset{parrow/.style={dotted}}
    \centering
    \begin{tikzpicture}
        \pgfmathsetmacro{\gridspace}{1.5}
        \pgfmathsetmacro{\xspace}{0.9}
        \node (a) {
            \begin{tikzpicture}
                \node [tnode] (xa) at (0,0) {\xa};
                \node [tnode] (za) at (1 * \xspace * \gridspace, 1 * \gridspace) {\za};
                \node [tnode] (xb) at (2 * \xspace * \gridspace, 0) {\xb}; 
                \node [tnode] (zb) at (3 * \xspace * \gridspace, 1 * \gridspace) {\zb};
                \draw [<-] (za) to (xa);
                \draw [<-] (za) to (xb);
                \draw [<-] (zb) to (xb);
            \end{tikzpicture}
        };
        
    \end{tikzpicture}
    \caption{Recognition model obtained from reversing the edges of the generative model}
    \label{naiveRecMod}
\end{figure}
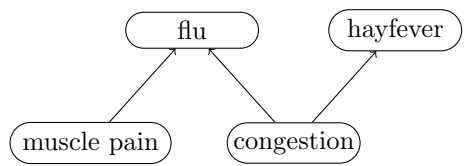
\noindent It is clear that when someone is congested, whether or not they have \xa \ affects the likelihood of that person having hayfever. If someone is congested and also has muscle pain, the congestion is more likely to be caused by the flu. On the other hand, the absence of muscle pain would make it more likely for the congestion to be caused by hayfever. This dependence is however not captured in the graph in Figure \ref{naiveRecMod}, because no information can flow from muscle pain to hayfever. By adding an edge between \xa \ and \zb, or between \za \ and \zb, this dependence can be captured. (Figure \ref{fig4}).\\

\begin{figure}[ht]
    \tikzset{tnode/.style={rounded rectangle, draw, minimum width = 2cm}}
    \tikzset{parrow/.style={dotted}}
    \centering
    \begin{tikzpicture}
        \pgfmathsetmacro{\gridspace}{1.5}
        \pgfmathsetmacro{\xspace}{0.9}
        \node (a) {
            \begin{tikzpicture}
                \node [tnode] (xa) at (0,0) {\xa};
                \node [tnode] (za) at (1 * \xspace * \gridspace, 1 * \gridspace) {\za};
                \node [tnode] (xb) at (2 * \xspace * \gridspace, 0) {\xb}; 
                \node [tnode] (zb) at (3 * \xspace * \gridspace, 1 * \gridspace) {\zb};
                \draw [<-] (za) to (xa);
                \draw [<-] (za) to (xb);
                \draw [<-] (zb) to (xb);
                \draw [->, dotted, red, thick] (xa) to (zb);
                \draw [->, dotted, red, thick] (za) to (zb);
            \end{tikzpicture}
        };
        
    \end{tikzpicture}
    \caption{Recognition model with optional arrows in red to capture the dependence between $\xa$ and $\zb$}
    \label{fig4}
\end{figure}
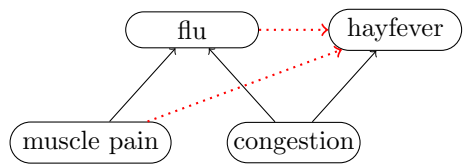

\section{Notation}

\subsection{Graph theory}
For a comprehensive overview of the theory and terminology of probabilistic graphical models, we refer to \citep{lauritzen1996,cowell1999,studeny2005probabilistic}. We define a \emph{graph} $G$ to be a pair $G = (N,E)$ where $N$ is a non-empty finite set of \emph{vertices} or \emph{nodes} and $E \subset N \times N$ such that $(s,s) \notin E$ for all $s \in N$. A graph $H = (A, \tilde{E})$ is called a \emph{subgraph} of $G$ if $A \subset N$ and $\tilde{E} \subset E$ and we write $H \subset G$. For a subset $A\subset N$, the \emph{vertex-induced subgraph} of $G$ is denoted $G[A]$ and is given by $(A, E[A])$, with $E[A] = \{(s,t) \in E : s,t \in A\}$. When both $(s,t)$ and $(t,s)$ are in $E$, we say that there is an \emph{undirected edge} between $s$ and $t$. 
When $(s,t) \in E$ and $(t,s) \notin E$ we say that there is a \emph{directed edge} going from $s$ to $t$ and write $s \to t$. 
We say that a graph is \emph{directed} if all edges are directed.\\ 

In the following we let $G = (N,E)$ be a directed graph. 
We say that two vertices $s,t \in N$ are \emph{joined} if there is an edge between the two. A set of vertices is called \emph{complete} if all pairs of its elements are joined. A \emph{path} in $G$ from $s$ to $t$ is a sequence of distinct nodes $s = u_0, ..., u_n = t$ such that $\left(u_i, u_{i+1}\right) \in E$ for all $i\in \{0 ,..., n-1\}$. 
A \emph{cycle} is a path of length $n > 1$ with the modification that the end points are identical. We say that a graph is \emph{acyclic} if it does not possess any cycles. A directed graph which is acyclic is called a \emph{directed acyclic graph,} or DAG.\\

In the following we let $G = (N,E)$ be a DAG. 
A \emph{trail} from $s$ to $t$ is a sequence of distinct nodes $s = u_0, ..., u_n = t$ such that $\left(u_i, u_{i+1}\right) \in E$ or $(u_{i+1}, u_i) \in E$ or both, for all $i\in \{0 ,..., n-1\}$. Note that movement along a trail could go against the direction of the arrows, in contrast to the case of a path. 
A \emph{loop} is a trail of length $n > 1$ with the modification that the end points are identical.   
If $(s,t) \in E$ we call $s$ a \emph{parent} of $t$ and $t$ a \emph{child} of $s$. The set of parents of a node $t$ is denoted $\pa_G(t)$ and the set of children of a node $s$ is denoted $\ch_G(s)$. 
A node $t \neq s$ is called a \emph{descendant} of $s$ if there exists a path from $s$ to $t$. The set of descendants of a node $s$ is denoted $\des_G(s)$. 
The set of \emph{non-descendants} of $s$ is given by $N \setminus (\{s\} \cup \des_G(s))$ and is denoted by $\nd_G(s)$.   
$G$ is called \textit{perfect} if for all $s$, the set $\pa_G(s)$ is complete. We let $\leaves(G)=\{s \in N : \ch_G(s)=\emptyset \} $ be the set of nodes without children, and $\roots(G) = \{ s \in N : \pa_G(s) = \emptyset \}$ be the set of nodes without parents (see Figure \ref{fig:subsetsN}). 
For $e = (s, t) \in E$, let $e^* = (t,s)$, $E^* = \{e^* : e \in E\}$, $G^* = (N, E^*)$ the graph $G$ with its edges reversed, $E^\sim = E \cup E^*, G^\sim = (N, E^\sim)$ the \emph{skeleton} (i.e.\ undirected version) of $G$, and $E_{\pa(G)}^\sim = \{(t_1,t_2) : \exists s \in N, t_1, t_2 \in \pa_G(s) \}, G^\rM = (N, E^\sim \cup E_{\pa(G)}^\sim)$ the \emph{moral graph} of $G$, which is the skeleton of $G$, with extra (undirected) edges between all parents of every vertex in $G$.  \\

For a trail $\gamma = (u_0, .., u_n)$, we let $\gamma^* = (u_n, ..., u_0)$ be its reversed version, and $(\gamma, s) = (u_0, .., u_n, s)$ for $s \notin \gamma$ such that $(u_n, s) \in E$ or $(s, u_n) \in E$ be its prolongation. Now let $\gamma_1 = (u_0, .., u_n)$ and $\gamma_2 = (v_0, .., v_m)$ be trails such that $u_n = v_0$. We define $\gamma = \gamma_1 ; \gamma_2$ to be the \emph{concatenation} of $\gamma_1$ and $\gamma_2$ such that $\gamma = (u_0, ..., u_{i}, s, v_{j}, ..., v_m)$, with $s$ the first node in $\gamma_1$ which belongs to $\gamma_{1} \cap \gamma_{2}$. Let $\gamma = (u_0, .., u_n)$ be a trail and $u_i$ a node on this trail that is not one of the endpoints, i.e.\ $0 < i < n$. $u_i$ is called a \emph{v-structure} if $u_{i-1} \rightarrow u_i \leftarrow u_{i+1}$. $\gamma$ is said to be \emph{blocked} by $S \subset N$ if $\gamma$ contains a vertex $u$ such that either: (a) $u \in S$ and $u$ is not a v-structure; (b) $u$ and $\des_G(u)$ are not in $S$ and $u$ is a v-structure. Note that when one of the endpoints of the trail is in $S$, the trail is definitely blocked, since endpoints can never be v-structures.
Let $A, B, S \subset N$ (not necessarily disjoint). $A, B$ are said to be \emph{d-separated} by $S$ if all trails from $A$ to $B$ are blocked by $S$ and we write $A \perp_G B \mid S$.
A \emph{topological ordering of $G$} is an injective map $\cO: N \to \{1,..., |N|\}$ that assigns to every node a number such that, if two nodes are joined, the edge points from the lower to the higher numbered node, i.e.\ $(s,t)\in E$ implies $\cO(s)<\cO(t)$. When $s,t \in N$ are such that $\cO(s) < \cO(t)$ we say $s$ is \textit{older} than $t$, and $t$ is \emph{younger} than $s$. Given a topological ordering $\cO$, the set of \emph{predecessors} of a node $s$, denoted $\pr^{\cO}(s)$, is the set of all nodes with a lower topological number\footnote{Note that some authors define the set of predecessors to be the set of nodes with a lower topological number in all possible topological orderings.}, i.e.\ $\pr^{\cO}(s) = \{t \in N : \cO(t) < \cO(s)\}$. Note that, although $\des_G(s)\cap \pr^{\cO}(s)=\emptyset$, the set $\pr^{\cO}(s)$ in general depends on the choice of topological ordering (see Figure \ref{figure:importanceTopologicalOrdering}).

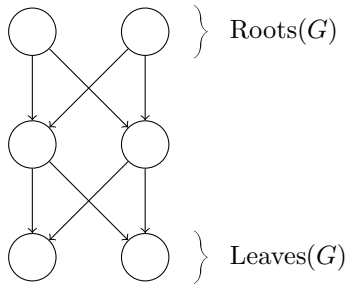
\begin{figure}[ht]
    \centering 
    \tikzset{parrow/.style={dotted}}
    \tikzset{mybrace/.style={decorate, decoration = {calligraphic brace, amplitude=5pt}}}
    \tikzset{mymbrace/.style={decorate, decoration = {calligraphic brace,mirror, amplitude=5pt}}}

    \begin{tikzpicture}
        \pgfmathsetmacro{\gridspace}{1.5}

        \node (a) at (0,0) [neuron] {};
        \node (b) at (1*\gridspace,0) [neuron] {};
        \node (c) at (0,1*\gridspace) [neuron] {};
        \node (d) at (1*\gridspace,1*\gridspace) [neuron] {};
        \node (e) at (0,2*\gridspace) [neuron] {};
        \node (f) at (1*\gridspace,2*\gridspace) [neuron] {};
        \draw[->] (e) to (c);
        \draw[->] (e) to (d);
        \draw[->] (f) to (c);
        \draw[->] (f) to (d);
        \draw[->] (d) to (a);
        \draw[->] (d) to (b);
        \draw[->] (c) to (a);
        \draw[->] (c) to (b);
        
        \pgfmathsetmacro{\hdistance}{.65}
        \pgfmathsetmacro{\vdistance}{.35}
        
        \draw [mymbrace] (\hdistance + \gridspace, 2 * \gridspace - \vdistance)  --  (\hdistance + \gridspace, 2 * \gridspace + \vdistance) node [pos=.5, right=10pt] {$\roots(G)$};
        \draw [mymbrace] (\hdistance + \gridspace, - \vdistance)  --  (\hdistance + \gridspace,  \vdistance) node [pos=.5, right=10pt] {$\leaves(G)$};

     \end{tikzpicture}
     \caption{Different subsets of $N$ for a graph $G$}
     \label{fig:subsetsN}
\end{figure}

\subsection{Probability on graphs}
Consider a DAG $G = (N,E)$. To every node $s \in N$ we associate a measurable space $(\sX_s, \cX_s)$. The state spaces are either real finite-dimensional vector spaces or finite sets and to each measurable space we associate a ($\sigma$-finite) base measure $\mu_s$ which is typically the Lebesgue measure or counting measure respectively. 
For a non-empty\footnote{When $A = \emptyset$ we let $\sX_\emptyset = \{()\} $ be the space containing only the empty sequence and $\cX = \{\emptyset, \sX_\emptyset\}$ be the trivial $\sigma$-algebra. $X_\emptyset$ is simply given by the constant map $X_\emptyset \colon \cX \ni x \mapsto () \in \sX_\emptyset$. See \cite{studeny2005probabilistic} Section 2.1 for details.}  
subset $A \subset N$ we let $\sX_A = \times_{s \in A} \sX_s$ be the Cartesian product of the individual $\sX_s$ and  $\cX_A = \otimes_{s \in A} \cX_s$ be the product $\sigma$-algebra. We write $(\sX, \cX) = (\sX_N, \cX_N)$ and assign to this space the base measure $\mu = \otimes_{s \in N} \mu_s$, which is the product measure. In this paper, we consider probability distributions $P$ over the space $(\sX, \cX)$. For every $s \in N$ we let $X_s\colon \sX \to \sX_s$ be the random variable projecting onto the individual spaces, and similarly $X_A =  (X_s)_{s \in A}$ and $X = X_N$. A typical element of $\sX_s$ is denoted $x_s$ with $x_A = (x_s)_{s\in A}$ and $x = (x_s)_{s\in N}$. We write $P_A$ for the distribution of $X_A$ on $(\sX_A, \cX_A)$, i.e.\ for $\sA \in \cX_A$, 
\begin{align}
    P_A(\sA) \coloneqq P\left(\left(X_A\right)^{-1}(\sA)\right) =  P\left(\sA \times \sX_{N \setminus A}\right).
\end{align}
For $A, C \subset N$, we say that a map $K \colon \cX_A \times \sX_C \to [0,1]$ is a \emph{Markov kernel} if
\begin{align}
    \sA &\mapsto K(\sA, x_C) \ \ &&\text{is a probability measure on $\cX_A$ for all $x_C \in \sX_C$},\\
    x_C &\mapsto K(\sA, x_C) \ \ &&\text{is $\cX_C$-measurable for all $\sA \in \cX_A$}.
\end{align}
Furthermore, we say that $K$ is a \emph{(regular) version of the conditional probability of $A$ given $C$} if it is Markov kernel and for all $\sA \in \cX_A, \sC \in \cX_C$
\begin{align}
    P_{A \cup C}(\sA \times \sC) = \int_\sC K(\sA, x_C) dP_C(x_C)
\end{align}
holds. It can be shown that in our setting, one can always find such a Markov kernel that is unique $P_C$-a.e. \citep{dudley2018real}. We therefore also denote such a Markov kernel by $P_{A \mid C}$. For $A, B, C \subseteq N$ we say that $A$ is \emph{conditionally independent} of $B$ given $C$ and write $A \indep B \mid C$ if for every $\sA \in \cX_A$ and $\sB \in \cX_B$
\begin{align}
    P_{A \cup B \mid C}(\sA \times \sB \mid x)=P_{A \mid C}(\sA \mid x) \cdot P_{B \mid C}(\sB \mid x) \quad P_C \text {-a.e.}.
\end{align}
For $s \in N$, a \emph{kernel function} will be a map $k^s(\cdot| \cdot)\colon \sX_s \times$ $\sX_{\pa_G(s)} \to \bR_{\geq 0}$ such that for all $x_{\pa_G(s)} \in \sX_{\pa_G(s)}$ 
\begin{align}
    \int k^s\left(x_s| x_{\pa_G(s)}\right) d\mu_s\left(x_s\right)=1.
    \end{align} 
A probability distribution $P$ is said to \textit{factorise} over a DAG $G$ if it has a density $p$ w.r.t.\ the product measure $\mu$ and there exist kernel functions $(k^s)_{s \in N}$ such that
\begin{align}
p(x)=\prod_{s \in N} k^s\left(x_s| x_{\pa_G(s)}\right).
\end{align}
We denote the set of probability distributions on $\sX$ that factorise over $G$ by $\cP^G$. Now let $A \subset N$ be such that its parents are themselves in $A$. A Markov kernel $K \colon \cX_{N \setminus A} \times \sX_{A} \to [0,1]$ is said to \textit{factorise} over a DAG $G$ if there exist kernel functions $(k^s)_{s\in N \setminus A}$ such that for every $x_{A} \in \sX_{A}$, $K\left((\cdot), x_{A}\right)$ has a density $p\left( (\cdot) | x_{A}\right)$, such that 
\begin{align}
    p\left(x_{N \setminus A} | x_{A}\right) = \prod_{s \in N \setminus A} k^s\left(x_s | x_{\pa_G(s)}\right).
\end{align}
We denote the set of such Markov kernels by $\cK^G$.

\section{Problem statement}

\begin{bff}[Goal I]
    Given a DAG $G = (N,E)$, find a DAG $G' = (N, E')$ such that $\roots(G') \supset \leaves(G)$ and for every $P \in \cP^G$, there exists a $K \in \cK^{G'}$ that is a version of the conditional distribution $P_{N\setminus \leaves(G)|\leaves(G)}$. 
\end{bff}
It turns out (Proposition \ref{connectVisibleNodes}) that this goal is equivalent (up to edges in $G'$ between nodes in $\leaves(G)$) to the following goal:

\begin{bff}[Goal II]
    Given a DAG $G = (N,E)$, find a DAG $G' = (N, E')$ such that $\cP^{G'} \supset \cP^G$ and the parents in $G'$ of the nodes in $\leaves(G)$ are themselves in $\leaves(G)$.
\end{bff}
Note that even though the parents in $G'$ of the nodes in $\leaves(G)$ are themselves in $\leaves(G)$, it is not necessarily the case that $\leaves(G)$ are oldest in every topological ordering of $G'$. See Figure \ref{figure:importanceTopologicalOrdering} for an example. However, there exists a topological ordering of $G'$ for which the nodes in $\leaves(G)$ precede the other nodes if and only if the parents in $G'$ of the nodes in $\leaves(G)$ are themselves in $\leaves(G)$.\\

In the remainder of the paper, we will focus on Goal II. Moreover, we sometimes impose the following extra condition:
\begin{equation} \label{conditionContainFlippedEdges}
    G' \supset G^*. 
\end{equation}
It can be argued that this is a natural condition since it enforces that the hierarchical structure of the generative model $G$ is preserved when finding a suitable $G'$.  Furthermore, note that $G^*$ satisfies the requirement that the parents in $G^*$ of $\leaves(G)$ are themselves in $\leaves(G)$. 

\begin{figure}[ht]
    \centering
    \begin{tikzpicture}
        \node (x) {
            \begin{tikzpicture}
                \node (a) at (0,0) [neuron] {};
                \node (b) at (-.7,1) [neuron] {};
                \node (c) at (.7,1) [neuron] {};
                \node (d) at (1.4, 0) [neuron] {};
                \node (g) at (-1.3, 1.8) {$G$};
                \draw[->] (b) to (a);
                \draw[->] (c) to (a);
                \draw[->] (c) to (d);
            \end{tikzpicture}
        };
        \node (y) [right=of x] {
            \begin{tikzpicture}
                \node (a) at (0,0) [neuron] {1};
                \node (b) at (-.7,1) [neuron] {3};
                \node (c) at (.7,1) [neuron] {4};
                \node (d) at (1.4, 0) [neuron] {2};
                \node (g) at (-1.3, 1.8) {$G'$};
                \draw[farrow] (b) to (a);
                \draw[farrow] (c) to (a);
                \draw[farrow] (c) to (d);
                \draw[->] (a) to (d);
                \draw[->] (b) to (c);
            \end{tikzpicture}
        };
        \node (z) [right=of y] {
            \begin{tikzpicture}
                \node (a) at (0,0) [neuron] {1};
                \node (b) at (-.7,1) [neuron] {2};
                \node (c) at (.7,1) [neuron] {4};
                \node (d) at (1.4, 0) [neuron] {3};
                \node (g) at (-1.3, 1.8) {$G'$};
                \draw[farrow] (b) to (a);
                \draw[farrow] (c) to (a);
                \draw[farrow] (c) to (d);
                \draw[->] (a) to (d);
                \draw[->] (b) to (c);
            \end{tikzpicture}
        };
    \end{tikzpicture}
    \caption{Pair of DAGs $G, G'$ that satisfy the second requirement of Goal II, but for which there exists a topological ordering of $G'$ (the one on the right) that does not reflect this.}
    \label{figure:importanceTopologicalOrdering}
    \end{figure}
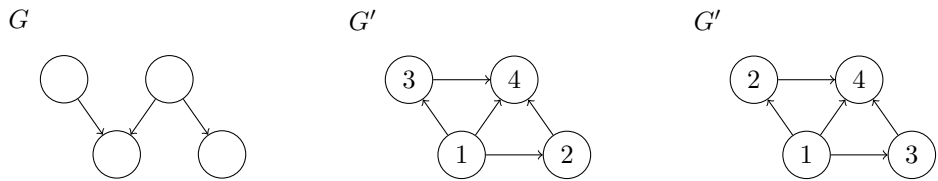

\section{Preliminaries}

\begin{lemma} \label{inclusionedgeset}
    Let $G=(N,E_G), H=(N,E_H)$ be DAGs such that $E_G \subset E_H$. Then $\cP^{G} \subset \cP^H$.
\end{lemma}
\begin{proof}
    Since $\pa_G(s) \subset \pa_H(s)$ for every node $s$, a density that can be written as $\prod_s k^s\left(x_s | x_{\pa_G(s)}\right)$ can also be written as $\prod_s k^s\left(x_s | x_{\pa_H(s)}\right)$.
\end{proof}

\begin{lemma}\label{5.14}
    (Theorem 5.14 in \citet{cowell1999}) Let $G$ be a DAG with a topological ordering $\cO$. For a probability distribution $P$ on $\sX$, the following conditions are equivalent:
    \begin{itemize}
        \item[(i)] $P \in \cP^G$,
        \item[(ii)] for all sets $A, B, S \subset N$ such that $A \perp_G B \mid S$ we have $A \indep B \mid S$ w.r.t.\ $P$,
        \item[(iii)] for all $s$ we have $s \indep \nd_G(s) \mid \pa_G(s)$ w.r.t.\  $P$,
        \item[(iv)] for all $s$ we have $s \indep \pr^{\cO}(s) \mid \pa_G(s)$ w.r.t.\  $P$.
    \end{itemize}
\end{lemma}

\begin{corollary}
    Let $\cO, \tilde{\cO}$ be two topological orderings of $G$. If $P$ satisfies property (iv) of Lemma \ref{5.14} w.r.t.\ $\cO$, then the same is true for $\tilde{\cO}$. 
\end{corollary}
\begin{proof}
    Note that $(i) - (iii)$ of Lemma \ref{5.14} are independent of the topological ordering. Therefore we have the following implications: for all $s$ we have $s \indep \pr^{\cO}(s) \mid \pa_G(s)$ w.r.t.\  $P$ $\implies$ $P \in \cP^G$ (with topological ordering $\cO$) $\implies$ $P \in \cP^G$ (with topological ordering $\tilde{\cO}$) $\implies$ for all $s$ we have $s \indep \pr^{\tilde{\cO}}(s) \mid \pa_G(s)$ w.r.t.\ $P$. 
\end{proof}

\section{Results}

\subsection{Equivalence of two goals}
\begin{proposition} \label{connectVisibleNodes}
    Let $G$ be a DAG. The following statements hold:
    \begin{itemize}
        \item[(a)] Let $G'=(N, E')$ be a DAG that satisfies Goal I and let  $\cO'$ be a topological ordering of $G'$. Then $\tilde{G}' = \left(N, E'\cup E_{\leaves}\right)$, with $E_{{\leaves}} = \{(s,t): s,t \in \leaves(G), \cO'(s) < \cO'(t)\}$  satisfies Goal II. 
        \item[(b)] Let $G' = (N, E')$ be a DAG that satisfies Goal II. Then $\tilde{G}' = \left(N, E' \setminus E_{\leaves}\right)$, with $E_{\leaves} = \{(s,t) \in E': s,t \in \leaves(G)\}$ satisfies Goal I. 
    \end{itemize}
\end{proposition}
The proof of this proposition is based on the following lemma, in which $G$ plays the role of $G'$ in the proposition. 
\begin{lemma} 
    Let $P$ be a distribution on $\sX$ that has a density $p$ w.r.t.\ $\mu$.
    \begin{itemize}
        \item[(a)] Let $G=(N,E)$ be a DAG with topological ordering $\cO$, $A \subset \roots(G)$, and $H = \left(N, E \cup E_{A}\right)$, with $E_{A} = \{(s,t): s,t \in A, \cO(s) < \cO(t)\}$. If there exists a Markov kernel $K \in \cK^{G}$ that is a version of the conditional distribution $P_{N\setminus A | A}$, then $P \in \cP^{H}$.
        \item[(b)] Let $G = (N, E)$  be a DAG, $A \subset N$ such that the parents of nodes in $A$ are themselves in $A$, and $H = \left(N, E \setminus E_{A}\right), E_{A} = \{(s,t) \in E: s,t \in A \}$. If $P \in \cP^{G}$, then there exists a Markov kernel $K \in \cK^{H}$ that is a version of the conditional distribution $P_{N\setminus A | A}$.
    \end{itemize}
    
\end{lemma}
\begin{proof}
    (a) Suppose $K$ is a version of $P_{N\setminus A|A}$ and $K \in \cK^G$. We need to show $P \in \cP^{H}$. We can write $p$ as follows:
    \begin{align}
        p(x) = p\left(x_{N\setminus A} | x_{A}\right) p\left(x_{A}\right),
    \end{align}
    where $p\left(x_{N\setminus A} | x_{A}\right)$ is the density corresponding to $K$ \citep{dudley2018real}. From the fact that $K \in \cK^G$ we know
    \begin{align}
        p\left(x_{N\setminus A} | x_{A}\right) = \prod_{s \in N\setminus A} k^s\left(x_s | x_{\pa_G(s)}\right).
    \end{align}
    Since all the nodes in $A$ are joined in $H$ we have
    \begin{align}
        p\left(x_{A}\right)  = \prod_{s \in A} k^s\left(x_s | x_{\pa_H(s)}\right).
    \end{align}
    Combining the above gives
    \begin{align}
        p(x) &= \prod_{s \in N\setminus A} k^s\left(x_s | x_{\pa_G(s)}\right) \prod_{s \in A} k^s\left(x_s | x_{\pa_H(s)}\right) \\
        &= \prod_{s \in N} k^s\left(x_s | x_{\pa_H(s)}\right), 
    \end{align}
    and therefore $P \in \cP^{H}$.    
    
    (b) Now suppose $P \in \cP^{G}$ and $x \in \sX$ such that $p\left(x_A\right) > 0$. We can write
    \begin{align}
        p(x) &= \prod_{s \in N} k^s\left(x_s | x_{\pa_G(s)}\right) \\
        &= \prod_{s \in N\setminus A} k^s\left(x_s | x_{\pa_G(s)}\right) \prod_{s \in A} k^s\left(x_s | x_{\pa_G(s)}\right) \\
        &= \prod_{s \in N\setminus A} k^s\left(x_s | x_{{\pa_{H}}(s)}\right) \prod_{s \in A} k^s\left(x_s | x_{\pa_G(s)}\right), 
    \end{align}
    where we can switch from $\pa_G$ to $\pa_{H}$ in the third equality because there are only edges removed between nodes in $A$ to obtain $H$. From the  definition of $A$ it follows that $\pa_G(s)\subset A$ for all $s \in A$, hence, $\prod_{s \in A} k^s\left(x_s | x_{\pa_G(s)}\right) = p\left(x_{A}\right)$. Dividing by $p\left(x_{A}\right)$ on both sides gives:
    \begin{align} \label{eq:conditionalFactorises}
        p\left(x_{ N\setminus A}|x_{A}\right) &= \prod_{s \in N\setminus A} k^s\left(x_s | x_{{\pa}_{H}(s)}\right). 
    \end{align}
    We know that there exists a Markov kernel $K$ that is a version of the conditional distribution of $N\setminus A$ given $A$ and that this kernel has density $p\left(x_{ N\setminus A}|x_{A}\right)$ \citep{dudley2018real}. Equation \eqref{eq:conditionalFactorises} shows that the density factorises and therefore $K \in \cK^H$. 
\end{proof}

\subsection{Conditions in terms of d-separation}
Necessary and sufficient conditions for our goal can be deduced from the following theorem:
\begin{theorem} \label{thmDSep}
    Let $G = (N, E), G'=(N, E')$ be DAGs, with $\cO'$ a topological ordering for $G'$. The following statements are equivalent:
    \begin{itemize}
        \item[(i)] $\cP^{G'} \supset \cP^G$,
        \item[(ii)] For all sets $A, B, S \subset N$ such that $A \perp_{G'} B \mid S$, we have $A \perp_G B \mid S$,
        \item[(iii)] For all $s \in N$, we have $s \perp_G \nd_{G'}(s) \mid \pa_{G'}(s)$,
        \item[(iv)] For all $s \in N$, we have $s \perp_G \pr^{\cO'}(s) \mid \pa_{G'}(s)$.
    \end{itemize}
\end{theorem}
\begin{proof}
    $(i) \implies (ii)$ (by contradiction) Suppose there exist $A, B, S$ such that $A \perp_{G'} B \mid S$, but $A \not\perp_G B \mid S$. From equation (5.7) in \citet{cowell1999} we know that there exists a $P \in \cP^G$ for which $A \not\indep B \mid S$. This violates (ii) of Lemma \ref{5.14} and therefore $P \notin \cP^{G'}$. \\
    $(ii) \implies (i)$ Let $P \in \cP^G$. We need to show $P \in \cP^{G'}$. From $(i) \implies  (ii)$ in Lemma \ref{5.14} we know that $A \perp_G B \mid S \implies A \indep B \mid S$ for $P$. Combining this with the assumption $A \perp_{G'} B \mid S \implies A \perp_G B \mid S$ gives $A \perp_{G'} B \mid S \implies A \indep B \mid S$. This means that $P$ satisfies (ii) of Lemma \ref{5.14} w.r.t.\ $G'$ and therefore $P \in \cP^{G'}$.\\
    $(i) \iff (iii)$ and $(i) \iff (iv)$ can be shown in a similar way. 
\end{proof}

\subsection{Conditions in terms of perfectness}

A sufficient condition for our goal can be deduced from the following theorem:

\begin{theorem}\label{sufficientCondition}
    Let $G = (N,E), G'=(N,E')$ be two DAGs. If $G'$ contains a subgraph $H$ such that $H$ is perfect and its undirected version $H^{\sim}$ contains the moral graph $G^{\rM}$, then $\cP^{G'} \supset \cP^{G}$.
\end{theorem}
\begin{proof}
    Let $P \in \cP^G$. By Lemma 5.9 from \citet{cowell1999} we know that $P$ factorises undirectedly\footnote{For a definition of this type of factorisation see Section 5.2 of \citet{cowell1999}} over the undirected graph $G^\rM$ and thus over any undirected graph $L = (N, E_L)$ containing $G^\rM$. From Proposition 5.15 in \citet{cowell1999} we know that $P$ factorises (directedly) over any perfect directed graph $H$ such that $H^\sim = L$. Thus, $P \in \mathcal{P}^{H}$, and by Lemma \ref{inclusionedgeset}, $P \in \mathcal{P}^{G^{\prime}}$. Therefore when $H^\sim \supset G^\rM$ we have $\cP^{G'} \supset \cP^{H} \supset \cP^{G}$.
\end{proof}

From this theorem we can conclude that if we reverse all the edges of $G$, fix a topological ordering such that $\leaves(G)$ are oldest, and then add edges consonant with this topological ordering until both $G'$ is perfect and $G'^{\sim} \supset G^{\rM}$, we obtain an inverse of $G$ that satisfies our goal. The example in Figure \ref{figure:imperfectGPrime} shows however that the condition that $G'$ needs to contain a perfect subgraph $H$ such that $H^\sim \supset G^\rM$ is not a necessary condition.

\begin{figure}[ht]
    \centering
\begin{tikzpicture}[]
    \node (a) at (0,0){
        \begin{tikzpicture}[scale=1.3]
            \node (a) at (0,0) [neuron] {};
            \node (b) at (-1,.3) [neuron] {};
            \node (c) at (1,.3) [neuron] {};
            \node (d) at (-.6, -.7) [neuron] {};
            \node (e) at (.6, -.7) [neuron] {};
            \node (g) at (-1.6, .8) {$G$};
            \draw[->] (b) to (a);
            \draw[->] (c) to (a);
            \draw[->] (b) to (d);
            \draw[->] (c) to (e);
        \end{tikzpicture}
    };
    \node (b) at (6,0) {
        \begin{tikzpicture}[scale=1.3]
            \node (a) at (0,0) [neuron] {};
            \node (b) at (-1,.3) [neuron] {};
            \node (c) at (1,.3) [neuron] {};
            \node (d) at (-.6, -.7) [neuron] {};
            \node (e) at (.6, -.7) [neuron] {};
            \node (g) at (-1.6, .8) {$G'$};
            \draw[farrow] (b) to (a);
            \draw[farrow] (c) to (a);
            \draw[farrow] (b) to (d);
            \draw[farrow] (c) to (e);
            \draw[->] (d) to (a);
            \draw[->] (e) to (a);
            \draw[->] (b) to (c);
            \draw[->] (e) to (b);
        \end{tikzpicture}
    };
\end{tikzpicture}
\caption{Pair of DAGs $G, G'$ that satisfies Goal II but $G'$ does not satisfy the condition in Theorem \ref{sufficientCondition}. One can check that $\cP^{G'} \supset \cP^G$ by verifying that Condition (iv) of Theorem \ref{thmDSep} is satisfied for all nodes.}
\label{figure:imperfectGPrime}
\end{figure}
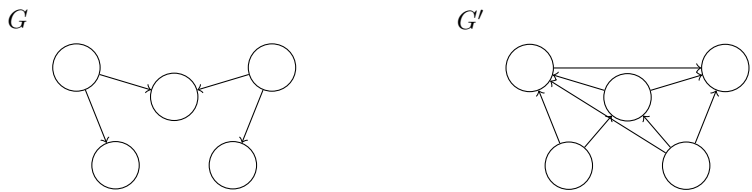

We do have the following necessary conditions on the graph $G'$ to satisfy our goal:

\begin{theorem} \label{necessaryCondition}
    Let $G, G'$ be DAGs. If $G'$ is such that $\cP^{G'} \supset \cP^G$ and $G' \supset G^*$, then $G'^\sim \supset G^\rM$ and for every $s$ in $N$, the vertex-induced subgraph $G'[\{s\} \cup \des_G(s)]$ contains a perfect subgraph $H_s$, such that $H^\sim_s \supset G^\rM[\{s\} \cup \des_G(s)]$. 
\end{theorem}

This theorem is based on the following two propositions. We will first prove both propositions and then show how Theorem \ref{necessaryCondition} can be obtained from it. 

\begin{proposition} \label{joinParentsInG}
    Let $\cP^{G'} \supset \cP^G$ and $G' \supset G^*$. Then $G'^\sim \supset G^\rM$. 
\end{proposition}
\begin{proof}
    Let $t_1, t_2 \in \pa_G(s)$ such that $t_1$ and $t_2$ are not joined in $G$ and assume WLOG $\cO'(t_2) < \cO'(t_1)$. 
    Assume for a contradiction that $t_{2} \notin \mathrm{pa}_{G^{\prime}}\left(t_{1}\right)$. By Condition (iv) of Theorem \ref{thmDSep} we must have that $t_1 \perp_G t_2 | \pa_{G'}(t_1)$. However, since $G' \supset G^*$, we know $s \in \pa_{G'}(t_1)$. Furthermore, $s$ is a v-structure on the trail $t_1, s, t_2$ in $G$. Therefore this trail is unblocked by $\pa_{G'}(t_1)$ and we arrive at a contradiction. Thus, necessarily $t_{2} \rightarrow t_{1}$ in $G^{\prime}$.
\end{proof}

\begin{proposition} \label{one-root-proposition}
    If $\cP^{G'} \supset \cP^G$, $G' \supset G^*$, and $|\roots(G)|=1$, then $G'$ contains a perfect subgraph $H$ such that $H^\sim \supset G^\rM$. 
\end{proposition}

\begin{proof}
    Our approach will be to construct a subgraph $H$ of $G'$, by starting from $G^*$ and iteratively adding edges. We show that every edge in $H$ is also in $G'$ and that the end result is both perfect and such that $H^\sim \supset G^\rM$. 

    \subsubsection*{Construction of $H$}

    Let $\cO'$ be a topological ordering of $G'$. We use an iterative process $H_i = (N, E_i)$ such that $H_{|N|} = H$. We start from $H_{-1} = G^*$. Then we join all parents of the same node in $G$ according to $\cO'$ to obtain $H_0$. In every subsequent step we join the parents in $H_i$ of the $(i+1)$th youngest node according to $\cO'$. 
    Formally this procedure can be defined as follows:

    \begin{align}
        E_{-1} &= E^* \label{prodecureFirst}\\
        E_0 &= E_{-1} \cup E_{\pa_G(G)} \\
        E_{i+1} &= E_i \cup E_{\pa_{H_i}(r_i)}, \ \ \ \text{for} \ 0 \leq i \leq |N| - 1,
    \end{align}
    with, 
    \begin{align}
        E_{\pa_G(G)} &= \{(t_1,t_2): \exists s \in N \ \text{such that} \ t_1,t_2 \in \pa_G(s) \ \text{and} \ \cO'(t_1) < \cO'(t_2) \} \\
        E_{\pa_{H_i}(r_i)} &=  \{(t_1,t_2): t_1, t_2 \in \pa_{H_i}(r_i), \cO'(t_1) < \cO'(t_2) \} \\
        r_i \in N \ &\text{such that} \ \cO'(r_i) = |N| - i.\label{prodecureLast}
    \end{align}
    See Figure \ref{fig:exampleCourse} for an example.

    \begin{figure}
        \centering
        \begin{tikzpicture}
            \node (row0) {
                \begin{tikzpicture}
                    \node (G) {
                        \begin{tikzpicture}
                            \node (r5) [neuron] at (0,  0) {}; 
                            \node (r3) [neuron] at (-.7,1) {}; 
                            \node (r4) [neuron] at (0.7,1) {}; 
                            \node (r2) [neuron] at (-.7,2.2) {}; 
                            \node (r1) [neuron] at (0.7,2.2) {}; 
                            \node (r0) [neuron] at (0,  3.2) {}; 
        
                            \node at (-1, 3.8) {$G$};
        
                            \draw [iarrow] (r0) to (r1);
                            \draw [iarrow] (r0) to (r2);
                            \draw [iarrow] (r1) to (r4);
                            \draw [iarrow] (r2) to (r3);
                            \draw [iarrow] (r3) to (r5);
                            \draw [iarrow] (r4) to (r5);
                        \end{tikzpicture}
                    };
                    \node (Gp) [right=of G]{
                        \begin{tikzpicture}
                            \node (r5) [neuron] at (0,  0)   {$1$}; 
                            \node (r3) [neuron] at (-.7,1)   {$3$}; 
                            \node (r4) [neuron] at (0.7,1)   {$2$}; 
                            \node (r2) [neuron] at (-.7,2.2) {$4$}; 
                            \node (r1) [neuron] at (0.7,2.2) {$5$}; 
                            \node (r0) [neuron] at (0,  3.2) {$6$}; 
        
                            \node at (-1, 3.8) {$G'$};
        
                            \draw [farrow] (r0) to (r1);
                            \draw [farrow] (r0) to (r2);
                            \draw [farrow] (r1) to (r4);
                            \draw [farrow] (r2) to (r3);
                            \draw [farrow] (r3) to (r5);
                            \draw [farrow] (r4) to (r5);
        
                            \draw [iarrow] (r2) to (r1);
                            \draw [iarrow] (r4) to (r2);
                            \draw [iarrow] (r4) to (r3);
                            \draw [iarrow] (r5) to (r0);
                        \end{tikzpicture}
                    };
                \end{tikzpicture}
            };
            \node (row1) [below=of row0] {
                \begin{tikzpicture}
                    \node (Hm1) {
                        \begin{tikzpicture}
                            \node (r5) [neuron] at (0,  0)   {$r_5$}; 
                            \node (r3) [neuron] at (-.7,1)   {$r_3$}; 
                            \node (r4) [neuron] at (0.7,1)   {$r_4$}; 
                            \node (r2) [neuron] at (-.7,2.2) {$r_2$}; 
                            \node (r1) [neuron] at (0.7,2.2) {$r_1$}; 
                            \node (r0) [neuron] at (0,  3.2) {$r_0$}; 
    
                            \node at (-1, 3.8) {$H_{-1}$};
    
                            \draw [farrow] (r0) to (r1);
                            \draw [farrow] (r0) to (r2);
                            \draw [farrow] (r1) to (r4);
                            \draw [farrow] (r2) to (r3);
                            \draw [farrow] (r3) to (r5);
                            \draw [farrow] (r4) to (r5);
    
                        \end{tikzpicture}
                    };
                    \node (H0) [right=of Hm1]{
                        \begin{tikzpicture}
                            \node (r5) [neuron] at (0,  0)   {$r_5$}; 
                            \node (r3) [neuron] at (-.7,1)   {$r_3$}; 
                            \node (r4) [neuron] at (0.7,1)   {$r_4$}; 
                            \node (r2) [neuron] at (-.7,2.2) {$r_2$}; 
                            \node (r1) [neuron] at (0.7,2.2) {$r_1$}; 
                            \node (r0) [neuron] at (0,  3.2) {$r_0$}; 
    
                            \node at (-1, 3.8) {$H_0$};
    
                            \draw [farrow] (r0) to (r1);
                            \draw [farrow] (r0) to (r2);
                            \draw [farrow] (r1) to (r4);
                            \draw [farrow] (r2) to (r3);
                            \draw [farrow] (r3) to (r5);
                            \draw [farrow] (r4) to (r5);
    
                            \draw [aarrow] (r4) to (r3);
                        \end{tikzpicture}
                    };
                    \node (H1) [right=of H0]{
                        \begin{tikzpicture}
                            \node (r5) [neuron] at (0,  0)   {$r_5$}; 
                            \node (r3) [neuron] at (-.7,1)   {$r_3$}; 
                            \node (r4) [neuron] at (0.7,1)   {$r_4$}; 
                            \node (r2) [neuron] at (-.7,2.2) {$r_2$}; 
                            \node (r1) [neuron] at (0.7,2.2) {$r_1$}; 
                            \node (r0) [neuron] at (0,  3.2) {$r_0$}; 
    
                            \node at (-1, 3.8) {$H_{1}$};
    
                            \draw [farrow] (r0) to (r1);
                            \draw [farrow] (r0) to (r2);
                            \draw [farrow] (r1) to (r4);
                            \draw [farrow] (r2) to (r3);
                            \draw [farrow] (r3) to (r5);
                            \draw [farrow] (r4) to (r5);
    
                            \draw [aarrow] (r2) to (r1);
                            \draw [iarrow] (r4) to (r3);
                        \end{tikzpicture}
                    };
                    \node (H2) [right=of H1]{
                        \begin{tikzpicture}
                            \node (r5) [neuron] at (0,  0)   {$r_5$}; 
                            \node (r3) [neuron] at (-.7,1)   {$r_3$}; 
                            \node (r4) [neuron] at (0.7,1)   {$r_4$}; 
                            \node (r2) [neuron] at (-.7,2.2) {$r_2$}; 
                            \node (r1) [neuron] at (0.7,2.2) {$r_1$}; 
                            \node (r0) [neuron] at (0,  3.2) {$r_0$}; 
    
                            \node at (-1, 3.8) {$H_{2}$};
    
                            \draw [farrow] (r0) to (r1);
                            \draw [farrow] (r0) to (r2);
                            \draw [farrow] (r1) to (r4);
                            \draw [farrow] (r2) to (r3);
                            \draw [farrow] (r3) to (r5);
                            \draw [farrow] (r4) to (r5);
    
                            \draw [iarrow] (r2) to (r1);
                            \draw [aarrow] (r4) to (r2);
                            \draw [iarrow] (r4) to (r3);
                        \end{tikzpicture}
                    };
                \end{tikzpicture}
            };
            \node (row2) [below=of row1]{
                \begin{tikzpicture}
                    \node (H3) {
                        \begin{tikzpicture}
                            \node (r5) [neuron] at (0,  0)   {$r_5$}; 
                            \node (r3) [neuron] at (-.7,1)   {$r_3$}; 
                            \node (r4) [neuron] at (0.7,1)   {$r_4$}; 
                            \node (r2) [neuron] at (-.7,2.2) {$r_2$}; 
                            \node (r1) [neuron] at (0.7,2.2) {$r_1$}; 
                            \node (r0) [neuron] at (0,  3.2) {$r_0$}; 
    
                            \node at (-1, 3.8) {$H_{3}$};
    
                            \draw [farrow] (r0) to (r1);
                            \draw [farrow] (r0) to (r2);
                            \draw [farrow] (r1) to (r4);
                            \draw [farrow] (r2) to (r3);
                            \draw [farrow] (r3) to (r5);
                            \draw [farrow] (r4) to (r5);
    
                            \draw [iarrow] (r2) to (r1);
                            \draw [iarrow] (r4) to (r2);
                            \draw [iarrow] (r4) to (r3);
                        \end{tikzpicture}
                    };
                    \node (H4) [right=of H3]{
                        \begin{tikzpicture}
                            \node (r5) [neuron] at (0,  0)   {$r_5$}; 
                            \node (r3) [neuron] at (-.7,1)   {$r_3$}; 
                            \node (r4) [neuron] at (0.7,1)   {$r_4$}; 
                            \node (r2) [neuron] at (-.7,2.2) {$r_2$}; 
                            \node (r1) [neuron] at (0.7,2.2) {$r_1$}; 
                            \node (r0) [neuron] at (0,  3.2) {$r_0$}; 
    
                            \node at (-1, 3.8) {$H_{4}$};
    
                            \draw [farrow] (r0) to (r1);
                            \draw [farrow] (r0) to (r2);
                            \draw [farrow] (r1) to (r4);
                            \draw [farrow] (r2) to (r3);
                            \draw [farrow] (r3) to (r5);
                            \draw [farrow] (r4) to (r5);
    
                            \draw [iarrow] (r2) to (r1);
                            \draw [iarrow] (r4) to (r2);
                            \draw [iarrow] (r4) to (r3);
                        \end{tikzpicture}
                    };
                    \node (H5) [right=of H4]{
                        \begin{tikzpicture}
                            \node (r5) [neuron] at (0,  0)   {$r_5$}; 
                            \node (r3) [neuron] at (-.7,1)   {$r_3$}; 
                            \node (r4) [neuron] at (0.7,1)   {$r_4$}; 
                            \node (r2) [neuron] at (-.7,2.2) {$r_2$}; 
                            \node (r1) [neuron] at (0.7,2.2) {$r_1$}; 
                            \node (r0) [neuron] at (0,  3.2) {$r_0$}; 
    
                            \node at (-1, 3.8) {$H_{5}$};
    
                            \draw [farrow] (r0) to (r1);
                            \draw [farrow] (r0) to (r2);
                            \draw [farrow] (r1) to (r4);
                            \draw [farrow] (r2) to (r3);
                            \draw [farrow] (r3) to (r5);
                            \draw [farrow] (r4) to (r5);
    
                            \draw [iarrow] (r2) to (r1);
                            \draw [iarrow] (r4) to (r2);
                            \draw [iarrow] (r4) to (r3);
                        \end{tikzpicture}
                    };
                \end{tikzpicture}
            };
        \end{tikzpicture}
        \caption{(\textbf{1\textsuperscript{st} row}) Example pair of DAGs $G, G'$ satisfying the requirements of Proposition \ref{one-root-proposition}. The fact that $\cP^{G'} \supset \cP^G$ can be  checked by verifying that Condition (iv) of Theorem \ref{thmDSep} is satisfied for all nodes. (\textbf{2\textsuperscript{nd} and 3\textsuperscript{rd} row}) Example course of the inversion procedure \eqref{prodecureFirst}--\eqref{prodecureLast} for the pair $G, G'$ depicted in top row. $H_{-1}$ is the version with the edges of $G$ reversed. In $H_0$ edges $E_{\pa_G(G)}$ for joining the parents in $G$ are added. For $H_{i+1}, i \geq 0$ the parents in $H_i$ of $r_{i}$ are joined, by adding $E_{\pa_{H_{i}}\left(r_{i}\right)}$ to the edge set. Note that because the parent set of $r_i$ in $H_i$ for $i \geq 2$ is already complete, no new edges are added in $H_3, H_4$ and $H_5$.}
        \label{fig:exampleCourse}
    \end{figure}
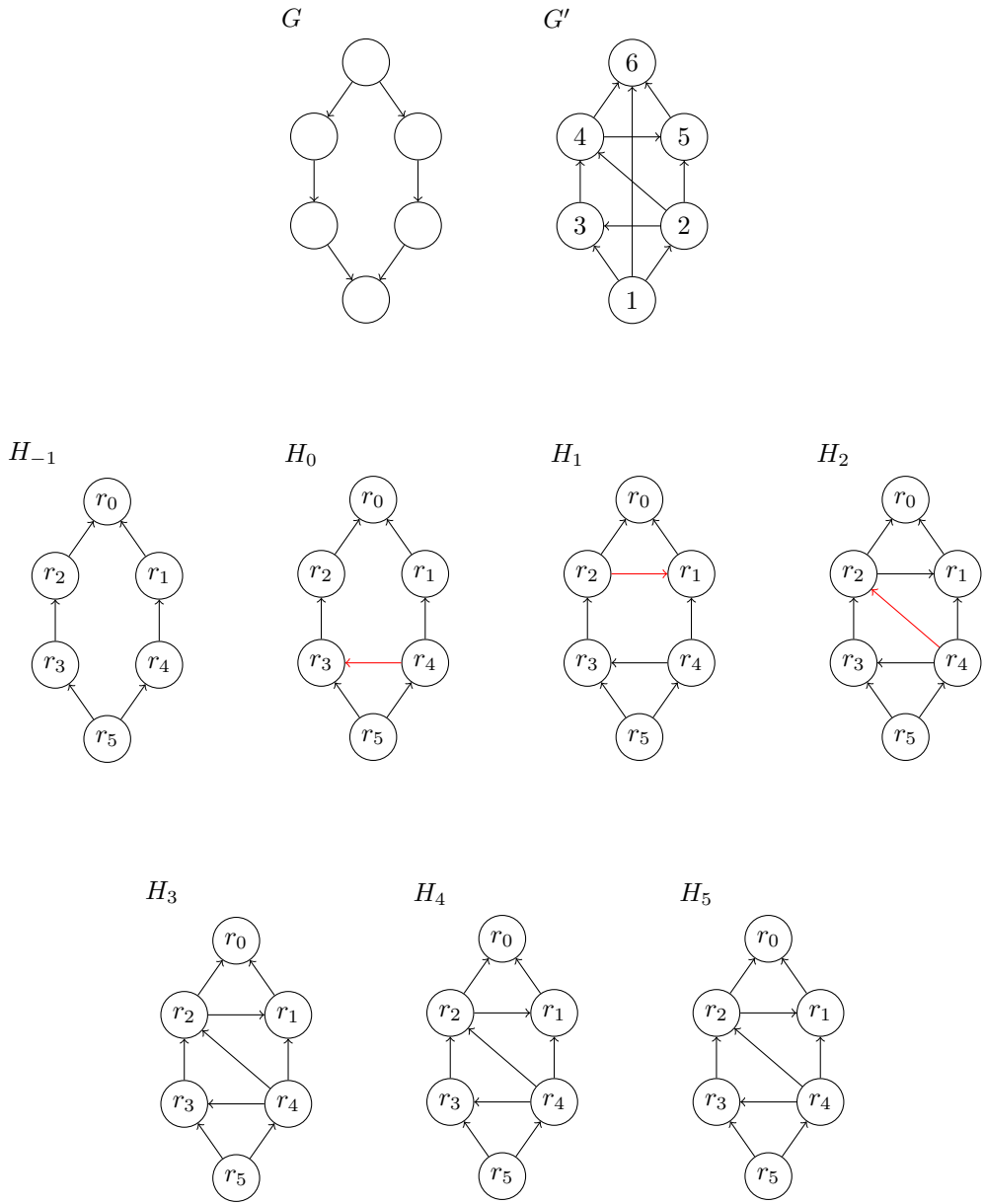

    \subsubsection*{$H_{-1}$ is a subgraph of $G'$}
    This follows directly from the fact that $G' \supset G^*$. 

    \subsubsection*{$H_{0}$ is a subgraph of $G'$}
    This follows from Proposition \ref{joinParentsInG} and  the fact that both $H_{0}$ and $G^{\prime}$ are consonant with $\mathcal{O}'$.

    \subsubsection*{$H_{i+1}$ is a subgraph of $G'$ for $0 \leq i \leq |N| - 1$}

    \begin{lemma} \label{case0}
        Let $\cP^{G'} \supset \cP^G$, $G' \supset G^*$, $|\roots(G)|=1$ and $\cO'$ a topological ordering for $G'$. Now let $r, u, v \in N$ such that $\cO'(v) < \cO'(u) < \cO'(r)$ and $r \to v$ in $G$. Then, $v \to u$ in $G'$. 
    \end{lemma}
    \begin{proof}
        Let $r_0$ be the youngest node in $\cO'$ and $\gamma_1$ be a path from $r_0$ to $u$ in $G$. Since $\cO'(u) > \cO'(v)$ we know $v \notin \gamma_1$. Now let $\gamma_2$ be a path in $G$ from $r_0$ to $r$. Since $\cO'(u) < \cO'(r)$ we know $u \notin \gamma_2$. Now let $\gamma = \gamma_2^* ; (\gamma_1, v)$ be a trail in $G$ from $u$ to $v$.  Since there are no v-structures on this trail and all nodes except $v$ are younger in $\cO'$ than $u$ it follows from property (iv) of Theorem \ref{thmDSep} that $v$ must be a parent of $u$ in $G'$.
    \end{proof}

    Suppose $(t_2, t_1) \in E_{\pa_{H_i}(r_i)}$, i.e.\ $t_2, t_1 \in \pa_{H_i}(r_i)$ and $\cO'(t_2) < \cO'(t_1)$. We need to show that $t_2 \to t_1$ in $G'$.  We distinguish two cases:\vspace{0.1cm}\\
    \textit{Case 1: There exists a $0 \leq j \leq i$ such that $r_j \to t_2$ in $G$.} \\
    This follows straight from Lemma \ref{case0} by setting $r = r_j, u = t_1, v = t_2$. \vspace{0.1cm}\\
    \textit{Case 2: For all $0 \leq j \leq i$ we have $r_j \not\to t_2$ in $G$.} \\
    See Figure \ref{figure:case2} for an illustration. Let $k = \min\{0 \leq j \leq i: (t_2, r_j) \in E_j\}$. Note that $(r_k, t_2) \in E_0$ follows from the definition of $k$ by contradiction. The definition of $H_{0}$ implies the existence of $s$ with $t_2 \rightarrow s\leftarrow r_{k}$ in $G$. The application of Lemma \ref{case0} with $r:=r_{k}, u = t_1$ and $v=s$ gives $s \in \pa_{G'}(t_1)$. Assume $t_2 \not\to t_1$ in $G'$ for a contradiction.  \vspace{0.1cm}\\
    We let $\gamma_1$ be a path in $G$ from $r_0$ to $t_1$, $\gamma_2$ a path in $G$ from $r_0$ to $r_k$ in $G$, and $\gamma = \gamma_2^* ; \gamma_1$ the concatenation of the reversed $\gamma_2$ and $\gamma_1$. Note that the trail $((\gamma, s), t_2)$ is not blocked by $\pa_{G'}(t_1)$, since for the only v-structure $(r_k, s, t_2)$ we have $s \in \pa_{G'}(t_1)$, and all other nodes on the path, except for $t_2$, are younger than $t_1$ in $\cO'$. However, by Condition (iv) of Theorem \ref{thmDSep} we know that this trail must be blocked and we arrive at a contradiction. Therefore we conclude  $t_2 \to t_1$ in $G'$.

    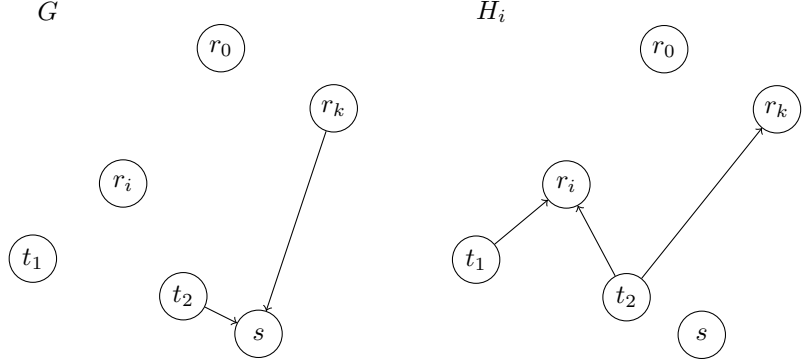
\begin{figure}[ht]
        \centering
        \begin{tikzpicture}
            \node (aa) {
                \begin{tikzpicture}
                    \node (t1) at (-.5,0) [neuron] {$t_1$};
                    \node (ri) at (0.7,1) [neuron]{$r_i$};
                    \node (t2) at (1.5,-0.5) [neuron] {$t_2$};
                    \node (s) at (2.5,-1) [neuron] {$s$};
                    \node (rj) at (3.5,2) [neuron] {$r_{k}$};
                    \node (r0) at (2,2.8) [neuron] {$r_0$};
    
                    \node at (-.3,3.3) {$G$};
    
                    \draw [iarrow] (t2) to (s);
                    \draw [iarrow] (rj) to (s);
    
                \end{tikzpicture}
            };
            \node (bb) [right=of aa] {
                \begin{tikzpicture}
                    \node (t1) at (-.5,0) [neuron] {$t_1$};
                    \node (ri) at (0.7,1) [neuron]{$r_i$};
                    \node (t2) at (1.5,-0.5) [neuron] {$t_2$};
                    \node (s) at (2.5,-1) [neuron] {$s$};
                    \node (rj) at (3.5,2) [neuron] {$r_{k}$};
                    \node (r0) at (2,2.8) [neuron] {$r_0$};
    
                    \node at (-.3,3.3) {$H_i$};
    
                    \draw [iarrow] (t1) to (ri);
                    \draw [iarrow] (t2) to (ri);
                    \draw [iarrow] (t2) to (rj);
                \end{tikzpicture}
            };
        \end{tikzpicture}
        \caption{Illustration of Case 2 in the proof of Proposition \ref{one-root-proposition}}
        \label{figure:case2}
    \end{figure}

    \subsubsection*{$H$ is perfect} 
    By adding the set $E_{\pa_{H_i}(r_i)}$ to the edge set, we join all the parents of $r_i$. After this step, no new parents of $r_i$ are created. We perform this step for every $i \in \{0,..., |N|-1\}$. Since $\{r_i:i=0,\dots,|N|-1\}= N$, the end result is perfect. 

    \subsubsection*{$H^\sim$ contains $G^\rM$} 
    This is ensured by adding the set $E_{\pa_G(G)}$. 

\end{proof}

\begin{remark} \label{rem:arbitraryTopologicalOrdering}
    Note that from the proof of Proposition \ref{one-root-proposition} it follows that the steps in the procedure described in \eqref{prodecureFirst}--\eqref{prodecureLast} are actually necessary for obtaining a graph $L$ that satisfies the following conditions: 
    \begin{itemize}
        \item[(i)] $\cP^L \supset \cP^G$,
        \item[(ii)] $L \supset G^*$,
        \item[(iii)] $\cO'$ is a topological ordering for $L$.
    \end{itemize}
    This means that the end result $H$ of the procedure will be a subset of any other graph $L$ satisfying conditions (i)-(iii). By Theorem \ref{sufficientCondition}, $H$ is also sufficient for satisfying these conditions.
    Furthermore, any other minimal inversion of $G$ can be obtained by using the same procedure \eqref{prodecureFirst}--\eqref{prodecureLast} but instead of $\cO'$ fixing a different topological ordering that is compatible with $G^*$. 
\end{remark}

\begin{remark}
    Since any perfect graph with a single leave has a unique topological ordering\footnote{This follows from the fact that every node has a unique youngest parent, and the single leave is the unique youngest node of the graph.}, it follows from the proposition that any DAG $G'$ for which there exists a DAG $G$ such that $\cP^{G'} \supset \cP^G$, $G' \supset G^*$ and $\roots(G) = 1$, has a unique topological ordering as well.  
\end{remark}

\begin{lemma} \label{lemmaVertexInducedSubgraph}
    Let $G, H$ be DAGs and $A \subset N$. If $G, H$ are such that $\cP^{G} \subset \cP^{H}$, then the same holds for the vertex-induced subgraph of both graphs: $\cP^{G[A]} \subset \cP^{H[A]}$.
\end{lemma}
\begin{proof}
    Let $\cO_{H}$ be a topological ordering for $H$ and $\cO_{H[A]}$ a topological ordering for $H[A]$ that preserves the order of $\cO_H$ and assume that $G, H$ are such that $\cP^{G} \subset \cP^{H}$. By Condition (iv) of Theorem \ref{thmDSep} we need to show that for all $s \in A$ we have $s \perp_{G[A]} \pr^{\cO_{H[A]}}(s) | \pa_{H[A]}(s)$. Therefore let $s \in A$ and $t \in \pr^{\cO_{H[A]}}(s)$ for which there exists a trail $\gamma$ in $G[A]$ from $s$ to $t$. We need to show that this trail is blocked by  $\pa_{H[A]}(s)$. Since Condition (iv) of Theorem \ref{thmDSep} holds for the original graphs $G$ and $H$, we know that this trail must be blocked by $\pa_{H}(s)$ in $G$. This implies that at least one of the following conditions must hold:
    \begin{itemize}
        \item[(i)] There is a v-structure $v$ on the trail such that neither $v$ nor its descendants in $G$ are in $\pa_{H}(s)$.
        \item[(ii)] There is a node $u$ on the trail that is not a v-structure and is in $\pa_{H}(s)$.
    \end{itemize}
   Since $\pa_{H[A]}(s) \subset \pa_{H}(s)$ and $\des_{G[A]}(v) \subset \des_{G}(v)$, the first condition remains valid. In case of the second condition, since all nodes on the trail are in $A$, $u$ must also be in $\pa_{H[A]}(s)$. Therefore, if the trail is blocked by $\pa_{H}(s)$ it is also blocked by $\pa_{H[A]}(s)$. 
\end{proof}

\begin{proof}[Proof of Theorem \ref{necessaryCondition}]
    The first part of the theorem follows directly from Proposition \ref{joinParentsInG}. Note that by Lemma \ref{lemmaVertexInducedSubgraph}, $\cP^{G'} \supset \cP^G$ implies $\cP^{G'[\{s\} \cup \des_G(s)]} \supset \cP^{G[\{s\} \cup \des_G(s)]}$, for any $s \in N$. Since $s$ is the unique root for $G[\{s\} \cup \des_G(s)]$, we know from Proposition \ref{one-root-proposition} that this implies that $G'[\{s\} \cup \des_G(s)]$ contains a perfect subgraph $H_s$, such that $H_s \supset G^\rM[\{s\} \cup \des_G(s)]$.  
\end{proof}

In practice, the inverse $G'$ is often obtained by simply inverting the edges in $G$. In this case we have the following necessary and sufficient condition to satisfy our goal. 

\begin{theorem} \label{primeIsStar}
    Let $G, G'$ be DAGs. If $G' = G^*$, then: $\cP^{G'} \supset \cP^{G} \iff \pa_G(s), \ch_G(s)$ are complete for all $s \in N$.
\end{theorem}
\begin{proof}
    $(\impliedby)$ If $\pa_G(s), \ch_G(s)$ are complete for all $s \in N$ and $G' = G^*$ this implies that $G'^\sim \supset G^\rM$ and $G'$ is perfect. The result now follows from Theorem \ref{sufficientCondition}. \\
    $(\implies)$ The fact that the parents much be complete follows from Proposition \ref{joinParentsInG}. Now assume by contradiction that there exists an $s \in N$ such that $u_1, u_2 \in \ch_G(s)$ are not joined. Consider the distribution $P \in \cP^G$ such that $X_{u_1}$ and $X_{u_2}$ are equal to $X_s$ and all other nodes (including $X_s$ itself) are independently distributed. Since $P$ must factorise over $G'$ we can write its density $p$ as
    \begin{align}
        p(x) = \prod_{s\in N} k^s\left(x_s|x_{\pa_{G'}(s)}\right).
    \end{align}
    Since $u_1$ and $u_2$ are both independent of their parents in $G'$, their kernels will be simply of the form $k^{u_1}\left(x_{u_1}\right), k^{u_2}\left(x_{u_2}\right)$ respectively. This however implies that $u_1$ and $u_2$ are themselves independent which contradicts the construction of $P$. 
\end{proof}

\begin{corollary}
    Let $G$ be a DAG. $\cP^{G} = \cP^{G^*}$ if and only if $\pa_G(s)$ and $\ch_G(s)$ are complete for all $s \in N$.
\end{corollary}
\begin{proof}
    This follows from twice applying Theorem \ref{primeIsStar} for $G=G$ and $G=G^*$ respectively and noting that completeness of $\pa_G(s)$ and  $\ch_G(s)$ implies completeness of $\pa_{G^*}(s)$ and $\ch_{G^*}(s)$.
\end{proof}

\subsection{Conditions in terms of single edge operations}

This paper discusses results related to the inversion of Bayesian networks in the sense of Goal II. In the proof of Proposition \ref{one-root-proposition}, we suggested an algorithm for inverting $G$, that starts by reversing all  edges of $G$ at once and then add edges where necessary. 
In this section we are looking at obtaining an inverse of $G$ by reversing the edges one by one, and potentially adding edges where necessary, based on a classical result often called Meek conjecture\footnote{It should be maybe more appropriately called Meek-Chickering Theorem, since it was proved in \citet{chickering2002optimal}.} stated below. An edge $(s,t)$ is called \emph{covered} when $\pa_G(t) = \pa_G(s) \cup \{s\}$. Note that reversing a covered edge preserves acyclicity of the graph. \\

\citet{meek1997graphical} states the following conjecture:
\begin{theorem}[Meek conjecture]
    Let $G=(N,E)$ and $G'=(N,E')$ be DAGs. $\cond$ if and only if there exists a sequence of DAGs $H_1, ..., H_n$ such that $H_1 = G'$ and $H_n = G$ and $H_{i+1}$ is obtained from $H_i$ by one of the following operations:
    \begin{itemize}
        \item[(a)] covered edge reversal,
        \item[(b)] edge removal.
    \end{itemize}  
\end{theorem}

This result suggests the outline of an algorithm for the inversion of a Bayesian network $G$. This algorithm starts with $G = L_1$. $L_{i+1}$ is obtained from $L_i$ by performing the inverse of operation (a) or (b), i.e.\ going from $L_{i+1}$ to $L_{i}$ can be done by applying operation (a) or (b). Since a covered edge remains covered after reversal, this operation is equal to its inverse. The inverse of removing an edge is simply adding an edge such that acyclicity is maintained. The goal is to reverse all edges present in $G$ without adding parents of $\leaves(G)$ that are not themselves in $\leaves(G)$. 

\subsection{Restricting the set of possible kernel functions}
The results derived in the above discuss the question what conditions $G'$ must satisfy such that for every $P \in \cP^G$, $\cK^{G'}$ contains a version of the conditional distribution $P_{N\setminus \leaves(G)|\leaves(G)}$. Here it is implied that we allow for all possible kernel functions $k^s$ in the definitions of $\cP^G$ and $\cK^{G'}$. In practice, however, restrictions are often put on the space of possible kernel functions. A common choice \citep{kingma2013auto} is to allow for only Gaussian kernel functions, of the form 
\begin{align}
    k^s\left(x_s | x_{\pa(s)};\theta_s\right) &\propto \exp\left(-\frac{1}{2} \left(\mu_{\theta_s}\left(x_{\pa(s)}\right) - x_s\right)^2\right)\\
    \mu_{\theta_s}\left(x_{\pa(s)}\right) &= f\left(\sum_{t \in \pa(s)} \theta_{(s;t)} x_t\right),
\end{align}
with $f$ some fixed possibly nonlinear function, and parametrised by the vectors $\theta_s = \{\theta_{(s,t)} \in \bR : t \in \pa(s)\}$. We will now investigate which results remain valid for the restricted case. Given a subset $R$ of kernel functions, we will denote the restricted spaces of probability distributions and Markov kernels factorising over $G$ by $\cP_R^G, \cK_R^G$ respectively. Before we dive into the results for general restrictions, we start by examining the case where $R$ is the set of Gaussian kernel functions defined above. Consider the pair of graphs $G, G'$ in Figure \ref{fig:simpleGraphs}. 
\begin{figure}[ht]
    \centering
    \begin{tikzpicture}
        \pgfmathsetmacro{\gspace}{1.3}
        \node (a) at (0,0) {
            \begin{tikzpicture}
                \node [neuron] (a) at (0,0) {s};
                \node [neuron] (b) at (0,\gspace) {t};
                \draw [->] (b) to (a);
                \node at (-.7, -.2) {$G$};
            \end{tikzpicture}
        };
        \node (b) [right=of a] {
            \begin{tikzpicture}
                \node [neuron] (a) at (0,0) {s};
                \node [neuron] (b) at (0,\gspace) {t};
                \draw [->] (a) to (b);
                \node at (-.7, -.2) {$G'$};
            \end{tikzpicture}
        };
    \end{tikzpicture}
    \caption{Pair of graphs $G, G'$}
    \label{fig:simpleGraphs}
\end{figure}
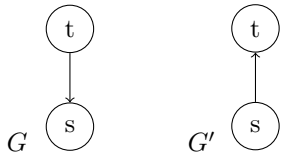
It is clear that this pair of graphs satisfies our original \text{Goal I}. However, when we restrict to the set Gaussian kernel functions, we are no longer able to model the posterior distribution exactly, as we will show now. Consider the distribution in $\cP^G_{R}$ given by
\begin{align}
    X_t &\sim \cN(0,1)\\
    X_s &\sim \cN\left(f\left(X_{t}\right), 1\right).
\end{align}
If the distribution $P_{t|s}$ would be in $\cK^{G'}_{R}$, we would need that the joint density of $X_t, X_s$ satisfies the following proportionality as a function of $x_t$
\begin{align}
    p(x_t, x_s) \propto \exp(-a x_t^2 + b x_t),
\end{align}
where only $b$ may depend on $x_s$. Working out the actual joint density gives 
\begin{align}
    p(x_t, x_s) \propto \exp(-\frac{1}{2}(x_t^2 + f(x_t)^2 + x_s f(x_t) + x_s^2)).
\end{align}
We can conclude that we only have that $P_{t|s} \in \cK^{G'}_{R}$ if $f$ is a linear function.\footnote{Note that this remains true even when we loosen the restriction and allow $\cK^{G'}_{R_{\tilde{f}}}$ to use a different function $\tilde{f}$, e.g. $\tilde{f} = f^{-1}$.}  From this example, we can conclude that the conditions that were sufficient for the unrestricted case, are in general not sufficient in the restricted case.\\

Now we look at the validity of the presented results under some general restrictions laid on the class of all kernel functions, i.e.\ $R$ can be any subclass of kernel functions. We start with the equivalence of the two goals, Proposition \ref{connectVisibleNodes}. Recall that the proposition shows that finding a $G'$ such that $\cP^{G'} \supset \cP^G$ and all parents in $G'$ of nodes in $\leaves(G)$ are themselves in $\leaves(G)$, is both a necessary (statement (a)) and sufficient condition (statement (b)) to satisfy Goal I. It is easy to see that statement (b) is still valid for the restricted case, i.e.it is still a sufficient condition. However for (a) we used that when all the nodes in $\roots(G)$ are connected, any density function can be written as $p\left(x_{\roots(G)}\right) = \prod_{s\in \roots(G)} k^s\left(x_s|x_{\pa_G(s)}\right)$. This is no longer the case when we restrict the space of possible kernel functions. We have that the condition is only necessary if for every $P \in \cP_R^G$, the marginal distribution $P_{\leaves(G)}$ factorises over a complete DAG of the leaves of $G$. A slightly weaker necessary condition for Goal I still holds in general, namely that for every subset $S \subset N\setminus \leaves(G)$ we need that $\cP^{G'[S]}_R \supset \cP^{G[S]}_R$. \\ 

For Theorem \ref{thmDSep}, note that conditions $(ii)-(iv)$ only relate to the graph structures of $G$ and $G'$. Therefore these conditions will still be equivalent for the restricted case. The implication $(ii) \implies (i)$ does not hold in general, which was exemplified by the Gaussian kernel functions above. The implication $(i) \implies (ii)$, on the other hand, does still hold, under the extra assumption that the restriction $R$ is such that for any graph $G$, for all $A, B, S \subset N$ such that $A \not\perp_G B \mid S$, there is a $P \in \cP_R^G$ for which $A \not\indep B \mid S$. We will sketch how this assumption is satisfied for the Gaussian kernel functions described above. Let $A, B, S \subset N$ such that $A \not\perp_G B \mid S$. This implies that there is a trail $\gamma = (u_0, ..., u_n)$ in $G$ from $u_0 = a \in A$ to $u_n = b \in B$ that is unblocked by $S$. We assume that the descendants of the v-structures on this path do not intersect the trail. If this is the case we replace the part of the original trail between the v-structure and the intersecting node by the trail going via the descendants of the former v-structure and keep doing this until all these type of v-structures are removed. Now let $\theta_{(s,t)} = 1$ for all $s=u_i, t=u_j$ with $u_i, u_j \in \gamma, u_j \in \pa_G(u_i), |j-i| = 1$ and all $s,t$ subsequent descendants of a v-structure on the trail and zero otherwise. It can be shown\footnote{As an example, one can consider a DAG consisting of a single trail $(a,...,b)$ of binary nodes for which the nodes without parents are sampled i.i.d. ($\mathrm{Bernoulli}(.5)$) and the rest of the nodes are the sum of their parents modulo 2. If we condition on the nodes that are v-structures (i.e.\ have parents on both sides) it is easy to see that $a$ and $b$ are not independent.} that for this distribution $a \not\indep b \mid S$ and therefore $A \not\indep B \mid S$. 
With this extra assumption we will now show $(i) \implies (ii)$. Suppose $A, B, S \subset N$ such that $A \perp_{G'} B \mid S$. This implies that for all $P \in \cP^{G'}_R$, we have $A \indep B \mid S$. Now suppose by contradiction that $A \not\perp_G B \mid S$. By the assumption, there must be a $P \in \cP^G_R$ for which $A \not\indep B \mid S$, which would contradict (a). Therefore $A \perp_G B \mid S$ which shows $(i) \implies (ii)$.  \\

Theorem \ref{sufficientCondition} is only a sufficient condition which is, by the Gaussian kernel function example, not sufficient any more in the restricted case. Theorem \ref{necessaryCondition} on the other hand is only a necessary condition. The proof of this theorem only uses the necessity of the conditions in Theorem \ref{thmDSep} which we showed above are still valid in the restricted case. We conclude that therefore Theorem 3 also still holds in the restricted case. \\

To conclude this section we summarise the results for the restricted case. We saw that we only have a slightly weaker necessary condition for Goal I, namely that for every subset $S \subset N\setminus \leaves(G)$ we need that $\cP^{G'[S]}_R \supset \cP^{G[S]}_R$. Necessary conditions for this latter condition are then provided by Theorem \ref{thmDSep} and \ref{necessaryCondition}, which are still valid for the restricted case.

\section*{Conclusion}
In this paper, we introduce some necessary and some sufficient conditions for the recognition network to be able to model the exact posterior distribution of a generative Bayesian network. In case that the generative network has a single root, the necessary and sufficient conditions coincide. However, for multiple roots there is still a gap between both conditions.

\subsection*{Further study directions}
A further direction of study could be to find a single necessary and sufficient condition for the general case. \\
Another interesting question is the following: ``What is the smallest number of edges in an inversion $G'$ of $G$?". Using the results on single edge operations, one could try to find an algorithm that finds an optimal inversion of $G$. \\
It is generally believed that the recognition network needs many edges to make exact modelling of the posterior distribution possible (Welling, personal communication, 2022). Therefore, the number of edges in the recognition network will be reduced to make it computationally efficient. In practice, this approximation does not seem to affect the quality of the inference. This phenomenon remains poorly understood, but very relevant to the computational side of machine learning. 

\section*{Acknowledgements}
The authors would like to thank the reviewers for helpful comments. JvO would like to thank Milan Studený and Martijn Oei for helpful discussions and comments. JvO and NA acknowledge the support of the Deutsche Forschungsgemeinschaft Priority Programme “The Active Self” (SPP 2134). Lastly, JvO and PvH would like to thank Floris Triest for providing a conducive working environment. 

\bibliography{references}

\begin{thebibliography}{19}
\providecommand{\natexlab}[1]{#1}
\providecommand{\url}[1]{\texttt{#1}}
\expandafter\ifx\csname urlstyle\endcsname\relax
  \providecommand{\doi}[1]{doi: #1}\else
  \providecommand{\doi}{doi: \begingroup \urlstyle{rm}\Url}\fi

\bibitem[Castelo and Ko\v{c}ka(2003)]{castelo2003inclusion}
R.~Castelo and T.~Ko\v{c}ka.
\newblock On inclusion-driven learning of {B}ayesian networks.
\newblock \emph{Journal of Machine Learning Research}, 4\penalty0 (Sep):\penalty0 527--574, 2003.

\bibitem[Chickering(2002)]{chickering2002optimal}
D.~M. Chickering.
\newblock Optimal structure identification with greedy search.
\newblock \emph{Journal of machine learning research}, 3\penalty0 (Nov):\penalty0 507--554, 2002.

\bibitem[Cowell et~al.(1999)Cowell, Dawid, Lauritzen, and Spiegelhalter]{cowell1999}
R.~G. Cowell, P.~Dawid, S.~L. Lauritzen, and D.~J. Spiegelhalter.
\newblock \emph{Probabilistic Networks and Expert Systems}.
\newblock Springer-Verlag New York, 1999.

\bibitem[Dayan et~al.(1995)Dayan, Hinton, Neal, and Zemel]{dayan1995helmholtz}
P.~Dayan, G.~E. Hinton, R.~M. Neal, and R.~S. Zemel.
\newblock The helmholtz machine.
\newblock \emph{Neural computation}, 7\penalty0 (5):\penalty0 889--904, 1995.

\bibitem[Dudley(2018)]{dudley2018real}
R.~M. Dudley.
\newblock \emph{Real analysis and probability}.
\newblock CRC Press, 2018.

\bibitem[Flesch and Lucas(2007)]{flesch2007markov}
I.~Flesch and P.~J. Lucas.
\newblock Markov equivalence in {B}ayesian networks.
\newblock In \emph{Advances in probabilistic graphical models}, pages 3--38. Springer, 2007.

\bibitem[Gershman and Goodman(2014)]{gershman2014amortized}
S.~Gershman and N.~Goodman.
\newblock Amortized inference in probabilistic reasoning.
\newblock In \emph{Proceedings of the annual meeting of the cognitive science society}, volume~36, 2014.

\bibitem[Kingma and Welling(2013)]{kingma2013auto}
D.~P. Kingma and M.~Welling.
\newblock Auto-encoding variational {B}ayes.
\newblock \emph{arXiv preprint arXiv:1312.6114}, 2013.

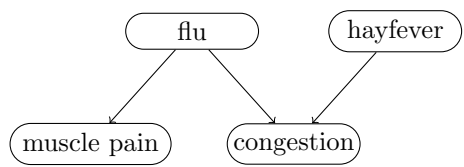
\bibitem[Koller and Friedman(2009)]{koller2009probabilistic}
D.~Koller and N.~Friedman.
\newblock \emph{Probabilistic graphical models: principles and techniques}.
\newblock MIT press, 2009.

\bibitem[Lauritzen(1996)]{lauritzen1996}
S.~L. Lauritzen.
\newblock \emph{Graphical models}, volume~17.
\newblock Clarendon Press, 1996.

\bibitem[Louizos et~al.(2017)Louizos, Welling, and Kingma]{louizos2017learning}
C.~Louizos, M.~Welling, and D.~P. Kingma.
\newblock Learning sparse neural networks through $ l\_0 $ regularization.
\newblock \emph{arXiv preprint arXiv:1712.01312}, 2017.

\bibitem[L{\"o}we et~al.(2022)L{\"o}we, Madras, Zemel, and Welling]{lowe2022amortized}
S.~L{\"o}we, D.~Madras, R.~Zemel, and M.~Welling.
\newblock Amortized causal discovery: Learning to infer causal graphs from time-series data.
\newblock In \emph{Conference on Causal Learning and Reasoning}, pages 509--525. PMLR, 2022.

\bibitem[Meek(1997)]{meek1997graphical}
C.~Meek.
\newblock \emph{Graphical Models: Selecting causal and statistical models}.
\newblock PhD thesis, PhD thesis, Carnegie Mellon University, 1997.

\bibitem[Molchanov et~al.(2019)Molchanov, Kharitonov, Sobolev, and Vetrov]{molchanov2019doubly}
D.~Molchanov, V.~Kharitonov, A.~Sobolev, and D.~Vetrov.
\newblock Doubly semi-implicit variational inference.
\newblock In \emph{The 22nd International Conference on Artificial Intelligence and Statistics}, pages 2593--2602. PMLR, 2019.

\bibitem[Pearl(1982)]{pearl1982reverend}
J.~Pearl.
\newblock Reverend {B}ayes on inference engines: A distributed hierarchical approach.
\newblock In \emph{Proceedings of the Second National Conference on Artificial Intelligence}, pages 133--136, 1982.

\bibitem[Studeny(2005)]{studeny2005probabilistic}
M.~Studeny.
\newblock \emph{Probabilistic Conditional Independence Structures}.
\newblock Information Science and Statistics. Springer London, 2005.

\bibitem[Verma and Pearl(1990)]{verma1990equivalence}
T.~Verma and J.~Pearl.
\newblock Equivalence and synthesis of causal models.
\newblock In \emph{Proceedings of the Sixth Annual Conference on Uncertainty in Artificial Intelligence}, pages 255--270, 1990.

\bibitem[Wainwright et~al.(2008)Wainwright, Jordan, et~al.]{wainwright2008graphical}
M.~J. Wainwright, M.~I. Jordan, et~al.
\newblock Graphical models, exponential families, and variational inference.
\newblock \emph{Foundations and Trends{\textregistered} in Machine Learning}, 1\penalty0 (1--2):\penalty0 1--305, 2008.

\bibitem[Webb et~al.(2018)Webb, Golinski, Zinkov, Rainforth, Teh, Wood, et~al.]{webb2018faithful}
S.~Webb, A.~Golinski, R.~Zinkov, T.~Rainforth, Y.~W. Teh, F.~Wood, et~al.
\newblock Faithful inversion of generative models for effective amortized inference.
\newblock \emph{Advances in Neural Information Processing Systems}, 31, 2018.

\end{thebibliography}

\end{document}